\documentclass[sigconf]{acmart}

\AtBeginDocument{%
  }

\copyrightyear{2025}
\acmYear{2025}
\setcopyright{acmlicensed}\acmConference[MMAsia '25]{ACM Multimedia
Asia}{December 9--12, 2025}{Kuala Lumpur, Malaysia}
\acmBooktitle{ACM Multimedia Asia (MMAsia '25), December 9--12, 2025,
Kuala Lumpur, Malaysia}
\acmDOI{10.1145/3743093.3770952}
\acmISBN{979-8-4007-2005-5/2025/12}

\usepackage{algorithm}
\usepackage{algpseudocode}
\usepackage{booktabs}
\usepackage{diagbox}
\usepackage{multirow}
\usepackage{tabularx}
\usepackage{pifont}
\usepackage{caption}
\usepackage{graphicx}
\usepackage{float} 
\usepackage{subcaption}
\usepackage{amsmath}
\usepackage{csquotes}

\begin{document}

\title{ChronoSelect: Robust Learning with Noisy Labels via Dynamics Temporal Memory}

\author{Jianchao Wang}
\affiliation{
  \institution{Shenzhen Institute for Advanced Study, University of Electronic Science and Technology of China}
  \city{Shenzhen}
  \country{China}
}
\email{202422280605@std.uestc.edu.cn}

\author{Qingfeng Li}
\affiliation{
    \institution{School of Computer Science and Engineering, University of Electronic Science and Technology of China}
    \city{Chengdu}
    \country{China}
}
\email{liqingfeng0316@gmail.com}

\author{Pengcheng Zheng}
\affiliation{
    \institution{School of Computer Science and Engineering, University of Electronic Science and Technology of China}
    \city{Chengdu}
    \country{China}
}
\email{zpc777@std.uestc.edu.cn}

\author{Xiaorong Pu}
\authornotemark[1]
\affiliation{
    \institution{School of Computer Science and Engineering, University of Electronic Science and Technology of China}
    \city{Chengdu}
    \country{China}
}
\affiliation{
    \institution{Shenzhen Institute for Advanced Study, University of Electronic Science and Technology of China}
    \city{Shenzhen}
    \country{China}
}
\email{puxiaor@uestc.edu.cn}

\author{Yazhou Ren}
\affiliation{
    \institution{School of Computer Science and Engineering, University of Electronic Science and Technology of China}
    \city{Chengdu}
    \country{China}
}
\affiliation{
    \institution{Shenzhen Institute for Advanced Study, University of Electronic Science and Technology of China}
    \city{Shenzhen}
    \country{China}
}
\email{yazhou.ren@uestc.edu.cn}


\renewcommand{\shortauthors}{Jianchao Wang, Qingfeng Li, Pengcheng Zheng, Xiaorong Pu*, Yazhou Ren}

\begin{abstract}
  Training deep neural networks on real-world datasets is often hampered by the presence of noisy labels, which can be memorized by over-parameterized models, leading to significant degradation in generalization performance. While existing methods for learning with noisy labels (LNL) have made considerable progress, they fundamentally suffer from static snapshot evaluations and fail to leverage the rich temporal dynamics of learning evolution. In this paper, we propose ChronoSelect (chrono denoting its temporal nature), a novel framework featuring an innovative four-stage memory architecture that compresses prediction history into compact temporal distributions. Our unique sliding update mechanism with controlled decay maintains only four dynamic memory units per sample, progressively emphasizing recent patterns while retaining essential historical knowledge. This enables precise three-way sample partitioning into clean, boundary, and noisy subsets through temporal trajectory analysis and dual-branch consistency. Theoretical guarantees prove the mechanism's convergence and stability under noisy conditions. Extensive experiments demonstrate ChronoSelect's state-of-the-art performance across synthetic and real-world benchmarks.
\end{abstract}

\ccsdesc[500]{Computing methodologies~Supervised learning}

\keywords{Semi-Supervised Learning; Noisy Label; Sample Selection; Robust Learning}

\maketitle

\section{Introduction}
Deep Neural Networks (DNNs) have become the cornerstone of modern artificial intelligence, achieving remarkable success across a wide range of applications. This success is largely driven by their ability to learn complex representations from large-scale datasets \cite{ref1, ref2}. However, the performance of these models is critically dependent on the quality of the training labels. In many real-world scenarios, collecting perfectly labeled data is prohibitively expensive or impractical. Consequently, datasets are often corrupted by noisy labels sourced from web scraping, crowd-sourcing, or other imperfect annotation methods \cite{ref3, ref4, ref5}. This presents a significant challenge, as the high capacity of over-parameterized DNNs allows them to easily memorize these incorrect labels \cite{ref6, ref7}, which severely degrades their generalization performance on unseen data \cite{ref8, ref9, ref10}. This phenomenon, known as the memorization effect, reveals a crucial learning dynamic: while networks can fit random noise, they tend to first learn simple, generalizable patterns from clean data before fitting noisy examples in later training stages \cite{ref7, ref11}. Capitalizing on this discovery, numerous recent methods \cite{ref31,ref43,ref44,ref45,ref29,ref52,ref53} have leveraged the memorization effect to develop robust learning techniques for noisy labels.

\textit{Sample selection} \cite{ref17,ref21,ref42,ref43} methods combat noisy labels by partitioning datasets into clean and noisy subsets, with performance critically dependent on selection criteria. Early \textit{small-loss strategies} \cite{ref17,ref10,ref21,ref29} are effective but limited: they require predefined thresholds or noise ratio knowledge, struggle to distinguish boundary samples, and fail to consider sample uncertainty. Although dynamic thresholds \cite{ref31} address some limitations, they introduce additional complexity. \textit{Fluctuation-based strategies} \cite{ref42,ref43,ref44,ref45} overcome these issues by analyzing historical predictions, identifying noisy samples through prediction volatility. These methods avoid threshold tuning and better detect boundary samples, but increase computational complexity compared to small-loss approaches. Both strategies share a critical limitation: they rely solely on current-round loss or predictions, ignoring temporal instability that accumulates training errors.

To resolve these persistent challenges, we introduce ChronoSelect, featuring a paradigm-shifting approach to temporal modeling. Our key insight stems from observing that learning evolves through distinct phases. Storing every prediction is computationally prohibitive and obscures developmental patterns under noise. Instead, we compress the entire prediction history into four strategic temporal stages that collectively capture the complete learning lifecycle. Our novel sliding update mechanism with controlled decay continuously refreshes these stages, progressively diminishing the influence of distant epochs while amplifying recent learning signals. This biologically-inspired forgetting mechanism enables adaptive tracking of sample characteristics while maintaining a stable knowledge foundation. Crucially, we prove this architecture converges to stable representations despite label noise, with perturbation effects decaying at $O(1/t)$, where $t$ refers to training epoch. By integrating these temporal patterns with dual-branch consistency metrics, ChronoSelect achieves unprecedented three-way partitioning accuracy, cleanly separating stable clean samples, evolving boundary cases, and fundamentally mislabeled examples.

Our core contributions are as follows:
\begin{itemize}
    \item We devise a novel four-stage temporal memory system that compresses prediction history into evolving distributions, updated via a unique sliding mechanism with controlled decay.
    \item We propose a temporal trajectory analysis enabling accurate three-way separation into clean, boundary, and noisy subsets without threshold tuning.
    \item We formally prove the memory convergence and stability under noisy conditions, with perturbation bounds established. 
    \item Superior results across multiple benchmarks demonstrate the framework's effectiveness and generalizability.
\end{itemize}

\section{Related Work}

\subsection{Memorization of DNNs}
A foundational challenge in training deep models on real-world data is their immense capacity. Early work \cite{ref6} robustly demonstrated that standard DNNs are capable of memorizing an entire dataset, even when the labels are assigned randomly. This finding indicated that traditional generalization theory was insufficient to explain the behavior of deep networks. A subsequent, seminal study by Arpit et al. \cite{ref7} provided a more nuanced picture, revealing the \enquote{memorization effect}. They empirically showed that DNNs do not memorize indiscriminately; instead, they prioritize learning simple, generalizable patterns from clean data during the initial training phases. Only as training progresses do the networks begin to fit the more complex and often contradictory patterns associated with noisy labels. This dynamic, where clean examples are learned first, has been further corroborated by studies showing that early learning can act as a form of regularization against noisy label memorization \cite {ref11}. The key insight from this line of research is the small-loss hypothesis: samples that consistently maintain a small training loss are highly likely to be correctly labeled \cite{ref7}. This principle has become a cornerstone for a vast number of LNL algorithms \cite{ref17, ref21}, providing a powerful, data-driven signal to differentiate clean from noisy samples.

\subsection{Handling Noisy Labels via Sample Selection}
Sample selection \cite{ref17, ref23, ref42, ref43, ref51} has emerged as a dominant strategy for learning with noisy labels. Early approaches like Co-teaching \cite{ref17} and Co-teaching+ \cite{ref18} established the small-loss hypothesis, identifying clean samples through low-loss values. These methods employ statistical models (Beta/Gaussian Mixture Models \cite{ref49, ref50, ref56}) to identify noisy samples through prediction volatility. While intuitive, these methods suffer from two fundamental limitations: they rely on crude binary partitioning that misclassifies ambiguous boundary samples near decision surfaces, and require careful threshold tuning that necessitates prior knowledge of noise ratios. Fluctuation-based techniques evolved to address these issues by incorporating historical prediction patterns, typically analyzing sequences of losses over multiple epochs. Self-Filtering \cite{ref42} differentiates noisy samples from boundary samples by detecting historical correct-to-misclassification transitions during training. Late Stopping \cite{ref43} removes samples with high FkL values as mislabeled examples. However, due to prohibitive storage requirements, most implementations only consider a limited window of $k$ recent epochs rather than full historical records.

ChronoSelect circumvents these constraints through a novel temporal compression approach. Our method captures complete historical learning patterns while maintaining only four compact memory units per sample, updated via a novel sliding mechanism. This design achieves comprehensive temporal analysis with minimal storage requirements. By preserving essential long-term dynamics often lost in window-based approaches, we enable precise three-way sample partitioning without thresholds or noise ratio estimation, while remaining practical for large-scale real-world datasets.

\begin{figure*}[htbp]
    \centering
	\begin{subfigure}{0.65\textwidth}
		\centering
		\includegraphics[width=\linewidth]{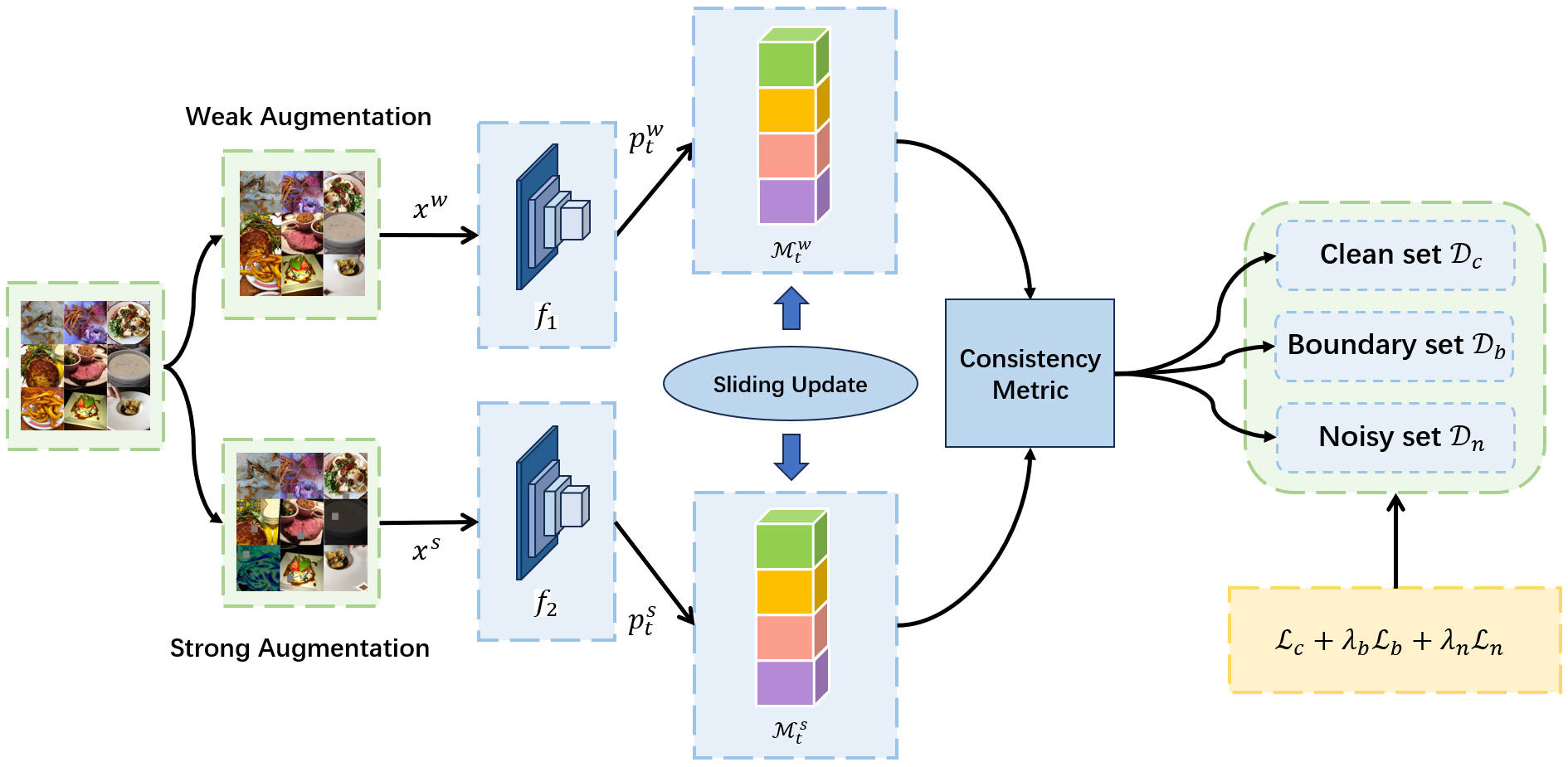}
        \caption{}
		\label{framework}
	\end{subfigure}
	\centering
	\begin{subfigure}{0.34\textwidth}
		\centering
		\includegraphics[width=0.8\linewidth]{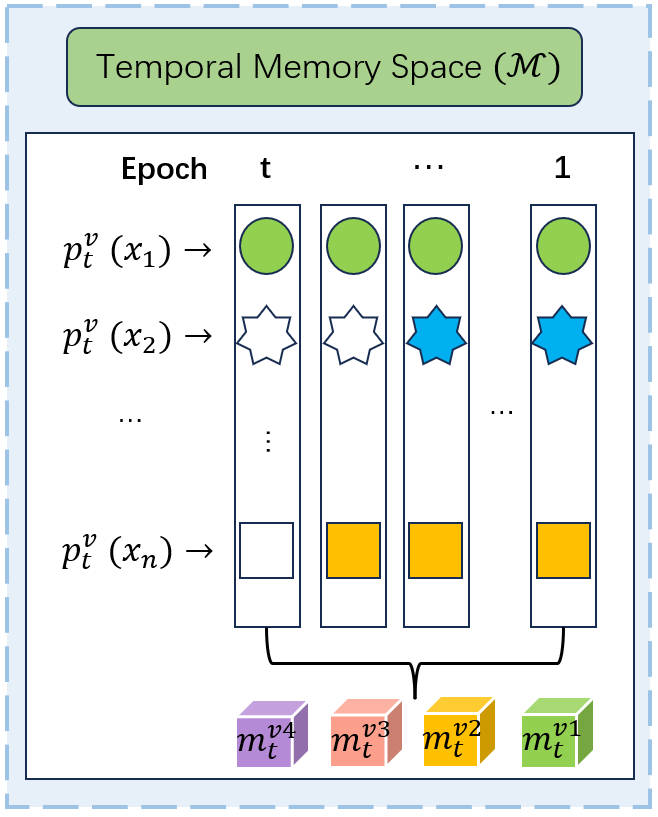}
        \caption{}
		\label{TMS}
	\end{subfigure}
	\centering
	\caption{(a) The framework of ChronoSelect. At each epoch, the ChronoSelect first employs a two-view-dual network for image classification, in which two views are obtained via two diverse augmentations. Then a dynamic sliding update strategy is designed to memorize the distribution of previous all epochs, which is used to divide the noisy data into tree subsets, i.e., clean set $\mathcal{D}_c$, boundary set $\mathcal{D}_b$ and noisy set $\mathcal{D}_n$ with consistency metric by dual network. Finally, we adopt different loss function to update the dual network. (b) The structure of Temporal Memory Space (TMS), which stores the historical prediction of each sample via sliding update.}
	\label{figure1}
\end{figure*}

\section{Proposed Method}
\subsection{Preliminaries}
Consider a classification task with input space $\mathcal{X} \subseteq \mathbb{R}^d$ and label space $\mathcal{Y} = \{1, \dots, K\}$. Given a label-corrupted training set $\widetilde{\mathcal{D}} = \{(x_i, \tilde{y}_i)\}_{i=1}^N$ where $x_i \in \mathcal{X}$ and $\tilde{y}_i \in \mathcal{Y}$ denote the input of the model and its corresponding noisy label, repectively, we aim to learn a classifier $f_\theta: \mathcal{X} \to \mathbb{P}(\mathcal{Y})$ parameterized by $\theta$. The output distribution at epoch $t$ is $p_t(x_i) = f(x_i; \theta^{(t)})$. The objective is to find optimal parameters $\theta^*$ that minimize the expected risk on clean data $\mathcal{D}$:
\begin{equation}
\theta^* = argmin_\theta \mathbb{E}_{(x,y)\sim\mathcal{D}}[\ell(f(x; \theta), y)]
\label{eq1}
\end{equation}
where $\ell(\cdot)$ is a loss function (e.g., cross-entropy).

To handle label noise, the \textit{sample selection} framework dynamically partitions $\widetilde{\mathcal{D}}$ at each epoch $t$:
\begin{equation}
\widetilde{\mathcal{D}} = \mathcal{S}^{(t)} \cup \mathcal{N}^{(t)}, \quad \mathcal{S}^{(t)} \cap \mathcal{N}^{(t)} = \emptyset
\label{eq2}
\end{equation}
The clean subset $\mathcal{S}^{(t)}$ is constructed via a selection function $\mathcal{G}$:
\begin{equation}
\mathcal{S}^{(t)} = \{ (x_i, \tilde{y}_i) \in \widetilde{\mathcal{D}} \mid \mathcal{G}(x_i, \tilde{y}_i; \theta^{(t)}) = 1 \}
\label{eq3}
\end{equation}
where $\mathcal{S}^{(t)}$, $\mathcal{N}^{(t)}$ represent clean subset and noise subset, repectively.
The model is then updated using only $\mathcal{S}^{(t)}$:
\begin{equation}
\theta^{(t+1)} = \theta^{(t)} - \eta \nabla_\theta \sum_{(x_i,\tilde{y}_i)\in\mathcal{S}^{(t)}} \ell(p_t(x_i), \tilde{y}_i)
\label{eq4}
\end{equation}
where $\eta$ denote the learning rate.

\subsection{Temporal Memory Space Construction and Sliding Update}
The key step in our selection strategy is to go through the total historical prediction stored in the dynamically updated memory space module. As shown in Figure \textcolor{red}{\ref{TMS}}, we collect all predictions of the training set $\mathcal{D}$ for all the previous epochs and store them in the Temporal Memory Space (TMS) $\mathcal{M}$ by the way that we've elaborately designed, which efficiently captures the evolutionary learning patterns of each sample. Unlike methods storing raw historical predictions, we maintain a compact yet expressive representation through four specialized memory units per sample per branch:
\begin{itemize}
    \item \textbf{Long-term Memory ($m^{v1}$):} Encodes foundational learning patterns from initial epochs.
    \item \textbf{Mid-term Memory ($m^{v2}$):} Captures accelerated learning during model refinement.
    \item \textbf{Short-term Memory ($m^{v3}$):} Reflects stabilization trends in later training phases.
    \item \textbf{Immediate Memory ($m^{v4}$):} Focuses on recent prediction refinements.
\end{itemize}

The memory space undergoes a sophisticated sliding update process that progressively transfers information between temporal stages while implementing controlled decay of older information. This biologically-inspired mechanism operates as a "computational forgetting" system, where each new epoch triggers a cascade of updates through the memory hierarchy.

Mathematically, this elegant transfer process is governed by:
\begin{equation}
    \begin{aligned}
        &m^{v1}_t = \frac{t}{t+1}m^{v1}_{t-1} + \frac{1}{t+1}m^{v2}_{t-1} && \text{(Foundation)} \\
        &m^{v2}_t = \frac{(t-1)}{t+1}m^{v2}_{t-1} + \frac{2}{t+1}m^{v3}_{t-1} && \text{(Transition)} \\
        &m^{v3}_t = \frac{(t-2)}{t+1}m^{v3}_{t-1} + \frac{3}{t+1}m^{v4}_{t-1} && \text{(Stabilization)} \\
        &m^{v4}_t = \frac{(t-3)}{t+1}m^{v4}_{t-1} + \frac{4}{t+1}p^v_t && \text{(Recent focus)}
    \end{aligned}
    \label{eq5}
\end{equation}

During initialization ($1 \leq t \leq 4$), we directly populate the memory units:
\begin{equation}
    m^{vj}_t(x_n) = p^v_t(x_n)
    \label{eq6}
\end{equation}
where $p^v_t(x_n)$ is the predicted probability distribution from the $v \in \{w(weak), s(strong)\}$ view branch at epoch $t$. This establishes baseline patterns for subsequent updates.

\subsection{Convergence and Stability Analysis}
We establish formal guarantees for our memory system's behavior under noisy conditions.

\begin{theorem}[Memory Convergence]
As training progresses ($t \to \infty$), each memory unit converges to the steady-state prediction $p^*$ when model predictions stabilize:
    \begin{equation}
        \lim_{t\rightarrow\infty} \mathbb{E}[\|m^{vj}_t - p^*\|] = 0 \quad \forall v,j
        \label{eq7}
    \end{equation}
\end{theorem}

\begin{proof}
Consider the immediate memory unit ($m^{i4}$) evolution:
\begin{equation}
    m^{v4}_t = \frac{(t-3)}{t+1}m^{v4}_{t-1} + \frac{4}{t+1}p^v_t
    \label{eq8}
\end{equation}
With prediction convergence $\|p^v_t - p^*\| \leq \epsilon_t \to 0$, define error $\delta^{v4}_t = m^{v4}_t - p^*$:
\begin{equation}
    \delta^{v4}_t = \frac{t-3}{t+1}\delta^{v4}_{t-1} + \frac{4}{t+1}(p^v_t - p^*)
    \label{eq9}
\end{equation}
The error norm satisfies:
\begin{equation}
    \|\delta^{v4}_t\| \leq \frac{t-3}{t+1}\|\delta^{v4}_{t-1}\| + \frac{4\epsilon_t}{t+1}
    \label{eq10}
\end{equation}
As $\frac{t-3}{t+1} \to 1$ and $\frac{4}{t+1} \to 0$, the error vanishes asymptotically, with similar convergence for other units via recursive analysis.
\end{proof}

\begin{theorem}[Memory Stability]
Prediction perturbations $\|p^v_t - p^*\| \leq \epsilon$ induce bounded memory deviations:
    \begin{equation}
        \max_{v,j} \|m^{vj}_t - p^*\| \leq \frac{4\epsilon}{t+1} + \mathcal{O}\left(\frac{1}{t^2}\right)
        \label{eq11}
    \end{equation}
\end{theorem}

\begin{proof}
By induction on the memory depth. For the base case ($t=4$), initialization gives $\|m^{v4}_4 - p^*\| = \|p^v_4 - p^*\| \leq \epsilon$. Assume the bound holds at epoch $k$, then at $k+1$:
\begin{equation}
    \|m^{v4}_{k+1} - p^*\| \leq \frac{k-3}{k+2}\underbrace{\|m^{vi4}_k - p^*\|}_{\leq \frac{4\epsilon}{k+1}} + \frac{4}{k+2}\epsilon \leq \frac{4\epsilon(k-2)}{(k+2)(k+1)} \leq \frac{4\epsilon}{k+2}
    \label{eq12}
\end{equation}
The $\mathcal{O}(1/t^2)$ residual emerges from higher-order Taylor expansions. Similar stability bounds propagate through all memory units.
\end{proof}

\noindent\textbf{Decay Dynamics Interpretation:}
The weight coefficients exhibit meaningful temporal properties:
\begin{equation}
    \alpha_j(t) = \frac{t+1-j}{t+1} \quad \beta_j(t) = \frac{j}{t+1} \quad \text{with} \quad \lim_{t\to\infty} \alpha_j(t) = 1, \lim_{t\to\infty} \beta_j(t) = 0
    \label{eq13}
\end{equation}
This design ensures:
\begin{itemize}
    \item Early training: Rapid adaptation to new information ($\beta$ dominates).
    \item Late training: Stability through historical preservation ($\alpha$ dominates).
    \item Positional awareness: Higher indices $j$ weight recent information more heavily.
\end{itemize}

\subsection{Sample Selection Strategy with the Temporal Memory Space}
ChronoSelect transforms temporal patterns into optimized supervision through an integrated framework that extracts meaningful signatures from memory, categorizes samples based on evolutionary patterns, and applies tailored supervision strategies. This creates a self-reinforcing system that progressively enhances noise robustness throughout training.

The foundation of our approach begins with extracting two complementary temporal signatures from the memory space that capture distinct aspects of learning behavior. The convergence metric $\Gamma_t(x_n)$ examines the loss trajectory across all four temporal phases, identifying whether a sample shows consistent improvement in prediction confidence over time. This critical metric is formally defined through the indicator function that checks for monotonic loss reduction across consecutive temporal stages: 
\begin{equation}
    \Gamma_t(x_n) = \mathbb{I}\left( \bigcap_{v \in \{w, s\}} \bigcap_{j=1}^3 \left( \mathcal{L}^{vj}_t > \mathcal{L}^{v(j+1)}_t \right) \right)
    \label{eq14}
\end{equation}
where $\mathbb{I}$ denotes the sign function outputs 0 or 1, and $\mathcal{L}^{ij}_t = \text{CE}(m^{ij}_{t-1}(x_n), y_n)$ represents the cross-entropy loss at temporal stage $j$ for branch $i$. Samples demonstrating this monotonic reduction across all phases indicate stable learning patterns where the model progressively refines its understanding.

Complementing this, the consistency metric $\psi(x_n)$ quantifies the agreement between dual-branch predictions throughout the learning journey, formally expressed as:
\begin{equation}
    \psi(x_n) = \frac{1}{4}\sum_{j=1}^4 \mathbb{I}\left( \arg\max m^{wj}_{t-1}(x_n) == \arg\max m^{sj}_{t-1}(x_n) \right)
    \label{eq15}
\end{equation}
Lower consistency values reveal fundamental uncertainty in sample interpretation that often indicates proximity to decision boundaries.

These temporal signatures enable precise categorization of samples into three distinct groups that reflect their fundamental relationship to the learning process. Clean samples represent the bedrock of reliable knowledge, exhibiting both perfect convergence patterns ($\Gamma_t(x_n) = 1$) and complete prediction consistency between branches ($\psi(x_n) = 1$). These samples typically show early and stable memorization patterns, forming the essential foundation for model generalization. Boundary samples show convergence but inconsistent predictions due to decision boundary proximity, where feature variations cause interpretation differences. Noisy samples display non-convergent, fluctuating loss trajectories signaling fundamental feature-label mismatches requiring specialized handling.

The formal partitioning rules are expressed as:
\begin{equation}
    \begin{aligned}
        \mathcal{D}_c &= \{ x_n : \Gamma_t(x_n) = 1 \wedge \psi(x_n) = 1 \} \\
        \mathcal{D}_b &= \{ x_n : \Gamma_t(x_n) = 1 \wedge \psi(x_n) < 1 \} \\
        \mathcal{D}_n &= \{ x_n : \Gamma_t(x_n) = 0 \}
        \label{eq16}
    \end{aligned}
\end{equation}

\subsection{Semi-Supervised Learning via Sample Selection for Noisy Labels}
The overall architecture is illustrated in Figure \textcolor{red}{\ref{framework}}. Capitalizing on the highly non-convex nature of deep networks \cite{ref27}, we implement a dual-branch framework \cite{ref17, ref29, ref18} where both branches share identical network structures but undergo differential initialization to enable complementary learning. Building on this foundation, specialized supervision strategies are deployed to align with each sample type's inherent characteristics and learning requirements.

For clean samples in $\mathcal{D}_c$, we employ standard cross-entropy loss to reinforce correct knowledge:
\begin{equation}
    \mathcal{L}_c = -\frac{1}{|\mathcal{D}_c|} \sum_{(x_n,y_n)\in\mathcal{D}_c} \sum_{k=1}^K y_n^k \log p^k(x_n)
    \label{eq17}
\end{equation}
where $p^k(x_n)$ denotes the predicted probability for class $k$, and $|D_c|$ the clean sample count.

This approach maximizes likelihood estimation for confidently labeled samples, strengthening the model's core understanding of unambiguous patterns through direct label supervision. The dual branches receive identical supervision for these samples, creating a consensus reinforcement effect that solidifies reliable knowledge.

Boundary samples in $\mathcal{D}_b$ receive fundamentally different treatment through Generalized Cross Entropy (GCE) loss with smoothness factor $q$ which is empirically set as 0.7 \cite{ref31}:
\begin{equation}
    \mathcal{L}_b = \frac{1}{|\mathcal{D}_b|} \sum_{(x_n,y_n)\in\mathcal{D}_b} \frac{1 - (p^{y_n}(x_n))^q}{q}
    \label{eq18}
\end{equation}

These samples lie near decision boundaries where label noise severely impacts model confidence. GCE's q-exponent mechanism balances the noise robustness of MAE and the optimization efficiency of cross-entropy, enabling learning from ambiguous samples without overcommitting to noisy labels, thus refining boundaries stably.

For noisy samples in $\mathcal{D}_n$, we implement sophisticated consistency regularization that transforms dual-network disagreement into a powerful supervisory signal. The loss function applies strong augmentations to create divergent views of the same sample:
\begin{equation}
    \mathcal{L}_n = \frac{1}{|\mathcal{D}_n|} \sum_{x_n\in\mathcal{D}_n} D_{\text{sym}} \left( p^1(\text{Aug}_1(x_n)) \| p^2(\text{Aug}_2(x_n)) \right)
    \label{eq19}
\end{equation}
where $D_{\text{sym}} \left( p^1(\text{Aug}_1(x_n)) \| p^2(\text{Aug}_2(x_n)) \right)$ indicates the symmetric KL divergence.

This elegant formulation treats dual networks as complementary teachers, converting noisy samples into unlabeled data. The symmetric divergence ensures balanced learning and self-correction, progressively improving sample interpretation without relying on corrupted labels.

The composite loss integrates all components:
\begin{equation}
    \mathcal{L}_{total}=\mathcal{L}_c+\lambda_b\mathcal{L}_b+\lambda_n\mathcal{L}_n
    \label{eq20}
\end{equation}
where $\lambda_b$ and $\lambda_n$ is the hyper-parameter of $\mathcal{L}_b$ and $\mathcal{L}_n$ to balance loss which is set as 1 and 0.1, respectively. The optimization procedure is detailed in Algorithm \textcolor{red}{\ref{alg:ChronoSelect}}.

\begin{algorithm}[t]
    \caption{The ChronoSelect method}
    \label{alg:ChronoSelect}
    \begin{algorithmic}[1]
    \Require Training data $\mathcal{D}=\{(x_n, y_n)\}_{n=1}^N$, dual-branch networks $f_1, f_2$ parameterized with $\theta_1$ and $\theta_2$, respectively
    \Require Initialize memory space $\mathcal{M}^w_0, \mathcal{M}^s_0$
    \Require Randomly initialize parameters $\theta_1, \theta_2$
    \For{epoch $t=1$ \textbf{to} $T$}
        \State \textbf{Memory Update:}
        \If{$t \leq 4$}
            \For{each sample $(x_n, y_n) \in \mathcal{D}$}
                \State Initialize $m^{vt}_t(x_n)$ by Eq. (\textcolor{red}{\ref{eq6}})
                \Comment{Initialize temporal}
            \EndFor
        \Else
            \For{$j=1$ \textbf{to} $4$}
                \State Update $m^{vj}_t(x_n)$ by Eq. (\textcolor{red}{\ref{eq5}})
                \Comment{Sliding window update}
            \EndFor
        \EndIf
        
        \State \textbf{Sample Selection:}
        \For{each sample $(x_n, y_n) \in \mathcal{D}$}
            \State Compute $\Gamma_t(x_n)$ by Eq. (\textcolor{red}{\ref{eq14}})
            \State Compute $\psi(x_n)$  by Eq. (\textcolor{red}{\ref{eq15}})
            \State Add $x_n$ into $\mathcal{D}_c$ or $\mathcal{D}_b$ or $\mathcal{D}_n$ by Eq. (\textcolor{red}{\ref{eq16}})
        \EndFor
        
        \State \textbf{Loss Computation:}
        \State Calculate $\mathcal{L}_c$, $\mathcal{L}_b$, $\mathcal{L}_n$ by Eq. (\textcolor{red}{\ref{eq17},\ref{eq18},\ref{eq19}})
        \State $\mathcal{L}_{total} = \mathcal{L}_c + \lambda_b\mathcal{L}_b + \lambda_n\mathcal{L}_n$ 
        \Comment{$\lambda_b=1, \lambda_n=0.1$}
    \EndFor
    \end{algorithmic}
\end{algorithm}

\begin{table*}[t]
    \centering
    \caption{Comparison with the SOTA methods on CIFAR-10 and CIFAR-100 with IDN. }
    \label{table1}
    \begin{tabular}{l c c c | c c c c c c c c}
        \toprule
            Dataset & \multicolumn{3}{c}{CIFAR-10} & \multicolumn{3}{c}{CIFAR-100}\\
            Noise rate & Inst. 20\% & Inst. 40\% & Inst. 60\% & Inst. 20\% & Inst. 40\% & Inst. 60\% \\
        \midrule
            $JoCoR$ \cite{ref32} & $88.78 \pm 0.15$ & $71.64 \pm 3.09$ & $63.46 \pm 1.58$ & $43.36 \pm 0.13$ & $23.95 \pm 0.44$ & $13.16 \pm 0.97$\\
            $DivideMix$ \cite{ref21} & $93.33 \pm 0.14$ & $95.07 \pm 0.11$ & $85.50 \pm 0.71$ & $79.04 \pm 0.21$ & $76.08 \pm 0.35$ & $46.72 \pm 1.32$\\
            $CORSES^2$ \cite{ref33} & $91.14 \pm 0.46$ & $83.67 \pm 1.29$ & $77.68 \pm 2.24$ & $66.47 \pm 0.45$ & $58.99 \pm 1.49$ & $38.55 \pm 3.25$\\
            $CAL$ \cite{ref34} & $92.01 \pm 1.25$ & $84.96 \pm 1.25$ & $79.82 \pm 2.56$ & $69.11 \pm 0.46$ & $63.17 \pm 0.19$ & $43.58 \pm 3.30$\\
            $CC$ \cite{ref26} & $93.68 \pm 0.12$ & $94.95 \pm 0.04$ & $94.55 \pm 0.11$ & $79.01 \pm 0.13$ & $76.84 \pm 0.19$ & $59.40 \pm 0.46$\\
            $DISC$ \cite{ref31} & $\underline{96.48} \pm \underline{0.04}$ & $\underline{95.94} \pm \underline{0.04}$ & $\underline{95.05} \pm \underline{0.05}$ & $\underline{80.12} \pm \underline{0.13}$ & $\underline{78.44} \pm \underline{0.19}$ & $\boldsymbol{69.57 \pm 0.14}$\\
            $ours$ & $\boldsymbol{96.78 \pm 0.08}$ & $\boldsymbol{96.14 \pm 0.05}$ & $\boldsymbol{95.12 \pm 0.10}$ & $\boldsymbol{80.34 \pm 0.04}$ & $\boldsymbol{78.62 \pm 0.14}$ & $\underline{69.37} \pm \underline{0.25}$\\
        \bottomrule
    \end{tabular}
\end{table*}

\begin{table}
    \centering
    \footnotesize
    \caption{Comparison with the SOTA methods on CIFAR-100 with sym. and asym. }
    \label{table2}
    \begin{tabular}{l c c  c c}
        \toprule
            Method & Sym 20\% & Sym 50\% & Sym 80\% & Asym 40\%\\
        \midrule
            $JoCoR$ \cite{ref32} & $53.01 \pm 0.04$ & $43.49 \pm 0.46$ & $15.49 \pm 0.98$ & $32.70 \pm 0.35$\\
            $DivideMix$ \cite{ref21} & $76.93 \pm 0.11$ & $74.26 \pm 0.12$ & $59.65 \pm 0.07$ & $55.56 \pm 0.03$\\
            $ELR+$ \cite{ref36} & $74.21 \pm 0.22$ & $65.01 \pm 0.44$ & $30.27 \pm 0.86$ & $73.73 \pm 0.34$\\
            $Co-learning$ \cite{ref37} & $66.58 \pm 0.15$ & $54.54 \pm 0.43$ & $35.45 \pm 0.79$ & $47.62 \pm 0.79$\\
            $GJS$ \cite{ref38} & $78.05 \pm 0.25$ & $72.58 \pm 0.25$ & $44.49 \pm 0.53$ & $63.70 \pm 0.22$\\
            $ACL$ \cite{ref47} & $73.41 \pm 0.22$ & $67.38 \pm 0.17$ & $55.75 \pm 0.36$ & $70.79 \pm 1.01$\\
            $ours$ & $\boldsymbol{79.72 \pm 0.14}$ & $\boldsymbol{75.18 \pm 0.05}$ & $\boldsymbol{58.82 \pm 0.29}$ & $\boldsymbol{76.30 \pm 0.20}$\\
        \bottomrule
    \end{tabular}
\end{table}

\section{Experimental}
\subsection{Controllable Noise Benchmarks}

\indent\textbf{Datasets.} 
We evaluate ChronoSelect on CIFAR-10 and CIFAR-100 \cite{ref39} under various label noise conditions. We consider three noise types: instance-dependent noise (IDN) generated following \cite{ref41} using truncated Gaussian-distributed class-specific noise rates, symmetric noise implemented through uniform label flipping to all possible classes \cite{ref21}; and asymmetric noise created by pairwise class flipping between semantically similar categories. Following established protocols \cite{ref26, ref40}, we test IDN noise rates from 20\% to 60\% on both datasets and symmetric noise ($\rho$ $\in$ \{20\%, 50\%, 80\%\}) and asymmetric noise ($\rho$ = 40\%) on CIFAR-100.\\
\textbf{Experimental Setup.} 
ResNet-18 \cite{ref55} serves as the backbone network trained for 200 epochs with batch size of 128. We employ Adam optimizer with initial learning rate 0.001, implementing linear decay to 0 after epoch 80 where $\beta_1$ decays from 0.9 to 0.1. Consistent with \cite{ref31}, two types of augmentation are utilized, i.e., weak and strong data augmentation. All experiments repeat three times with different random seeds, with results reported as mean±std test accuracy. Implementations use PyTorch 1.13.1 executing on a single NVIDIA GeForce RTX 4060ti GPU.\\
\textbf{Results and discussions.} 
Tables \textcolor{red}{\ref{table1}} and \textcolor{red}{\ref{table2}} demonstrate the superiority of ChronoSelect across diverse noise settings on CIFAR benchmarks. On CIFAR-10 with instance-dependent noise (Table \textcolor{red}{\ref{table1}}), our method achieves state-of-the-art results at all noise levels, outperforming DISC by margins of up to 0.20\%. For the more challenging CIFAR-100 benchmark, ChronoSelect obtains the highest accuracy at 20\% and 40\% noise rates, while attaining competitive performance at 60\% noise, confirming that the sliding memory update balances historical stability with new information assimilation.

Further evaluations in Table \textcolor{red}{\ref{table2}} confirm our robustness to both symmetric and asymmetric noise. ChronoSelect achieves the highest accuracy in 20\%/50\% symmetric noise and 40\% asymmetric noise, with only marginal degradation at extreme 80\% symmetric noise. Crucially, our method maintains stable performance across noise distributions without requiring prior knowledge of noise types – a significant advantage over methods like DivideMix \cite{ref21} that rely on Gaussian Mixture Models for noise estimation. The consistent improvements demonstrate ChronoSelect's ability to mitigate error accumulation through temporal dynamics analysis, particularly valuable in high-noise regimes.

\begin{table}
    \centering
    \caption{Comparison with the SOTA methods on WebVision.}
    \label{table3}
    \begin{tabular}{l c c | c c}
        \toprule
            Dataset & \multicolumn{2}{c}{WebVision} & \multicolumn{2}{c}{ILSVRC12}\\
            Accuracy (\%) & Top-1 & Top-5 & Top-1 & Top-5\\
        \midrule
            $ELR$ \cite{ref36} & 76.26 & 91.26 & 68.70 & 87.80\\
            $DivideMix$ \cite{ref21} & 77.32 & 91.64 & 75.20 & 90.84\\
            $ELR+$ \cite{ref36} & 77.78 & 91.68 & 70.29 & 89.76\\
            $GJS$ \cite{ref38} & 77.99 & 90.62 & 74.33 & 90.33\\
            $CC$ \cite{ref26} & 79.36 & $\boldsymbol{93.64}$ & 76.08 & $\boldsymbol{93.86}$\\
            $CSOT$ \cite{ref48} & 79.67 & 91.95 & 76.64 & 91.67\\
            $ours$ & $\boldsymbol{80.34}$ & $\underline{93.08}$ & $\boldsymbol{77.76}$ & $\underline{93.45}$\\
        \bottomrule
    \end{tabular}
\end{table}

\subsection{Real-world Noise Benchmarks}

\textbf{Datasets.} 
We evaluate on WebVision-1.0 \cite{ref4} (2.4M images with real-world noise), following standard protocols \cite{ref46} using its 50-class subset. Performance is assessed on both WebVision's validation set and ILSVRC 2012's clean validation set \cite{ref1} to measure generalization. The noise rate of WebVision is about 30\%.\\
\textbf{Experimental Setup.} 
Following \cite{ref46}, we adopt Inception-ResNetV2 \cite{ref54} as the backbone for WebVision without ImageNet pretraining. The model is trained for 50 epochs using Adam optimizer with an initial learning rate of 0.001, decayed by half every 10 epochs.\\
\textbf{Results and discussions.}
As shown in Table \textcolor{red}{\ref{table3}}, ChronoSelect establishes state-of-the-art performance on real-world noisy datasets. Our method achieves the highest Top-1 accuracy on both WebVision (80.34\%) and ILSVRC12 (77.56\%), surpassing CSOT by 0.67\% and 1.12\% respectively. While CC exhibits optimal Top-5 accuracy, ChronoSelect attains competitive sub-optimal results (93.08\% on WebVision and 93.45\% on ILSVRC12). Notably, the performance gap between WebVision and its subset ILSVRC12 remains minimal (2.78\% in Top-1), demonstrating exceptional generalization capability without overfitting to dataset-specific noise patterns. This consistency across domains highlights our method's advantage in reducing shortcut learning effects through dual augmentations.

\begin{table}
    \centering
    \caption{Ablation study on CIFAR-10/100 under IDN 20\%, 40\% and 60\%.}
    \label{table4}
    \begin{tabular}{c c c | c c | c c}
        \toprule
            \multicolumn{3}{c}{Subsets} & \multicolumn{2}{c}{CIFAR-10} & \multicolumn{2}{c}{CIFAR-100}\\
            $\mathcal{D}_c$ & $\mathcal{D}_b$ & $\mathcal{D}_n$ & Inst. 20\% & Inst. 40\% & Inst. 20\% & Inst. 40\%\\
        \midrule
            \ding{55} & \ding{55} & \ding{55} & 83.53 & 68.03 & 63.50 & 49.06\\
            \ding{51} & \ding{55} & \ding{55} & 94.11 & 92.45 & 74.58 & 71.17\\
            \ding{51} & \ding{55} & \ding{51} & 94.73 & 93.29 & 76.82 & 73.89\\
            \ding{51} & \ding{51} & \ding{51} & $\boldsymbol{96.78}$ & $\boldsymbol{96.09}$ & $\boldsymbol{80.34}$ & $\boldsymbol{78.32}$\\
        \bottomrule
    \end{tabular}
\end{table}

\subsection{Ablation Study}
Table \textcolor{red}{\ref{table4}} quantifies each sample subset's contribution. The first row shows baseline performance without splitting, using cross-entropy on all samples as clean. Using only clean samples ($\mathcal{D}_c$, \ding{51}\ding{55}\ding{55}) gives 94.11\%, validating temporal identification. Adding noisy samples ($\mathcal{D}_n$, \ding{51}\ding{55}\ding{51}) improves to 94.73\%, demonstrating semi-supervised benefits. Full inclusion of boundary samples ($\mathcal{D}_b$, \ding{51}\ding{51}\ding{51}) peaks at 96.78\%, a 2.05\%gain over $\mathcal{D}_c+\mathcal{D}_n$, proving synergy: three-level partitioning with differentiated supervision enables dynamic discrimination, reduces misclassification, and balances fidelity-robustness, maintaining 96.09\%at 40\%IDN while alternatives degrade.

\section{Conclusion}
In this paper, a temporal memory-based sample selection framework is proposed to solve the problem of sample misjudgment in label noise scenarios. The historical prediction information is recorded through the two-branch network, and the joint discriminant criterion of time series loss trend and prediction consistency is constructed. The samples are divided into three categories: clean, boundary and noise, and a differentiated supervision strategy is designed. Theoretical analysis proves the convergence and noise robustness of the memory storage mechanism. Experiments show that the method can still maintain a high purity rate under a high proportion of noise, and the generalization ability of the model is improved by the effective use of boundary samples. Compared with the traditional methods based on small loss or fixed threshold, the three-level partition criterion reduces the subjective assumption dependence and provides a more adaptive solution for complex noise scenes.

\begin{acks}
    This work was supported in part by Shenzhen Science and Technology Program (Nos. JCYJ20230807120010021 and JCYJ20230807115959041), National Natural Science Foundation of China (No. 62476052), Sichuan Science and Technology Program (Nos. 2024NSFSC1473 and 2024ZYD0268), and Sichuan Provincial Natural Science Foundation (No. 2024NSFSC0520).
\end{acks}

\bibliographystyle{ACM-Reference-Format}
\bibliography{reference}

\end{document}